\documentclass[letterpaper, 10 pt, conference]{ieeeconf}  % Comment this line out if you need a4paper

\IEEEoverridecommandlockouts                              % This command is only needed if 
\usepackage{amsmath}
\usepackage{amssymb}
\usepackage{dsfont}
\usepackage{color}
\usepackage{subcaption}
\usepackage{graphicx}

\usepackage{theorem}

\usepackage{booktabs}
\usepackage{multirow}

\usepackage{bm}

\usepackage{setspace}

\title{\LARGE \bf
Signal Temporal Logic Neural Predictive Control
}

\usepackage{xcolor}
\newcommand{\rbt}[1]{\textcolor{black}{#1}}

\author{Yue Meng and Chuchu Fan% <-this % stops a space
\thanks{*This work was supported by the Ford Motor Company and C3.ai Digital Transformation Institute.}% <-this % stops a space
\thanks{Yue Meng and Chuchu Fan are with the Massachusetts Institute of Technology, Cambridge, MA 02139 USA (email: mengyue@mit.edu; chuchu@mit.edu).}%
}

\newcommand{\partitle}[1]{\noindent{\textbf{#1:}}}

\newtheorem{theorem}{Theorem}
\newtheorem{problem}{Problem}

\begin{document}
\maketitle
\thispagestyle{empty}
\pagestyle{empty}
\begin{abstract}
Ensuring safety and meeting temporal specifications are critical challenges for long-term robotic tasks. Signal temporal logic (STL) has been widely used to systematically and rigorously specify these requirements. However, traditional methods of finding the control policy under those STL requirements are computationally complex and not scalable to high-dimensional or systems with complex nonlinear dynamics. \rbt{Reinforcement learning (RL) methods can learn the policy to satisfy the STL specifications via hand-crafted or STL-inspired rewards, but might encounter unexpected behaviors due to ambiguity and sparsity in the reward.} In this paper, we propose \rbt{a method to} directly learn a neural network controller to satisfy the requirements specified in STL. Our controller learns to roll out trajectories to maximize the STL robustness score in training. In testing, similar to Model Predictive Control (MPC), the learned controller predicts a trajectory within \rbt{a} planning horizon to ensure the satisfaction of the STL requirement in deployment. A backup policy is \rbt{designed} to ensure safety when our controller fails. Our approach can adapt to various initial conditions and environmental parameters. We conduct experiments on six tasks, \rbt{where our method with the backup policy outperforms the classical methods (MPC, STL-solver), model-free and model-based RL methods in STL satisfaction rate, especially on tasks with complex STL specifications while being 10X-100X faster than the classical methods}.
\end{abstract}
\section{Introduction}
% there should be a better word to describe temporal safety or rule constraints

% \chuchu{Many people raise the question ``Given the stochastic and numerical nature of gradient descent its impossible to empirically achieve exactly zero loss during training." We should keep this in mind and have it covered somehow. Maybe we could make the ``backup" MILP controller more explicit.}
Learning to control a robot to satisfy long-term and complex safety requirements and temporal specifications is critical in autonomous systems and artificial intelligence. For example, the vehicles should make a complete stop before entering the intersection with a stop sign, wait a few seconds, and then drive through it if no other cars are there. And a robot navigating through obstacles to reach the destination should always reach and stay at a charging station for a while to get charged when it is low on battery. 

However, designing the controller to satisfy those specifications is challenging. Traditional rule-based methods often require expert knowledge and several rounds of trial and error to find the best design to handle the problem. Other learning-based approaches either learn from demonstrations or rewards to find the control policy to satisfy the behavior specification. Those methods need plenty of expert data or great effort in reward design (an improper reward will result in a learned policy to generate unexpected behaviors)

To let the controller satisfy the exact behaviors, another direction of solving this problem is to describe the requirements in Signal Temporal Logic (STL) and find out the feasible plan by solving an online optimization problem. \rbt{Mixed-integer linear programming (MILP)~\cite{sun2022multi} is proposed to handle simple dynamics and easy-to-evaluate atomic propositions, which has exponential complexity and is hard to solve}. For more complicated systems, gradient-based~\cite{dawson2022robust} and sampling-based methods (such as Cross-Entropy Method (CEM) and covariance matrix adaptation evolution strategy (CMA-ES)~\cite{kapoor2020model}) are proposed to synthesize controllers to maximize the STL robustness score (which measures how well the STL is satisfied). However, they still need to solve the problem online for each initial state condition and each scenario, which limits their usage for more general cases.

Motivated by the line of work in robustness score~\cite{donze2010robust} and controller synthesis~\cite{dawson2022robust}, we propose Signal Temporal Logic Neural Predictive Control, which learns a Neural Network \rbt{(NN)} controller to generate \rbt{STL-satisfied trajectories}. The ``STL solving" process is conducted in training, and in the test phase, we use the trained controller (potentially with a backup policy) to roll out trajectories. \rbt{Thus}, we do not need the heavy online optimization/searching required for those gradient-based or sampling-based methods. 

The whole pipeline is as follows: We construct an \rbt{NN} controller and sample from the initial state (which might include environment information) distribution. In training, we roll out trajectories using the NN controller. We evaluate the approximated robustness score on those trajectories to maximize the robustness score. In testing, we follow the MPC procedure: our learned controller predicts a trajectory that is safe/rule-satisfied in the short-term horizon, and we pick the first (or first several) actions in the deployment. Finally, when the learned controller is detected violating the STL constraints, a sampling-based backup policy is triggered to guarantee the robot's safety.
%In this way, we guarantee the safety or rule-correctness in the long term.

We conduct experiments on \rbt{six tasks shown in Fig.~\ref{fig:screenshot}:} driving near intersections, reach-and-avoid problem, safe ship control, safe ship tracking control, robot navigation and manipulation. On tasks with simple dynamics or simple STL, our approach is on par with the RL methods in terms of STL accuracy and we surpass traditional methods such as MPC and STL solvers. We achieve the highest STL accuracy on hard tasks such as ship safe tracking control and robot navigation tasks, $20\%\sim40\%$ higher than the second best approach. \rbt{Our training time is similar to RL, and our inference speed is 1/10-1/100X faster than classical methods.}

\begin{figure*}[htbp]
\centering
    \includegraphics[width=1.0\textwidth]{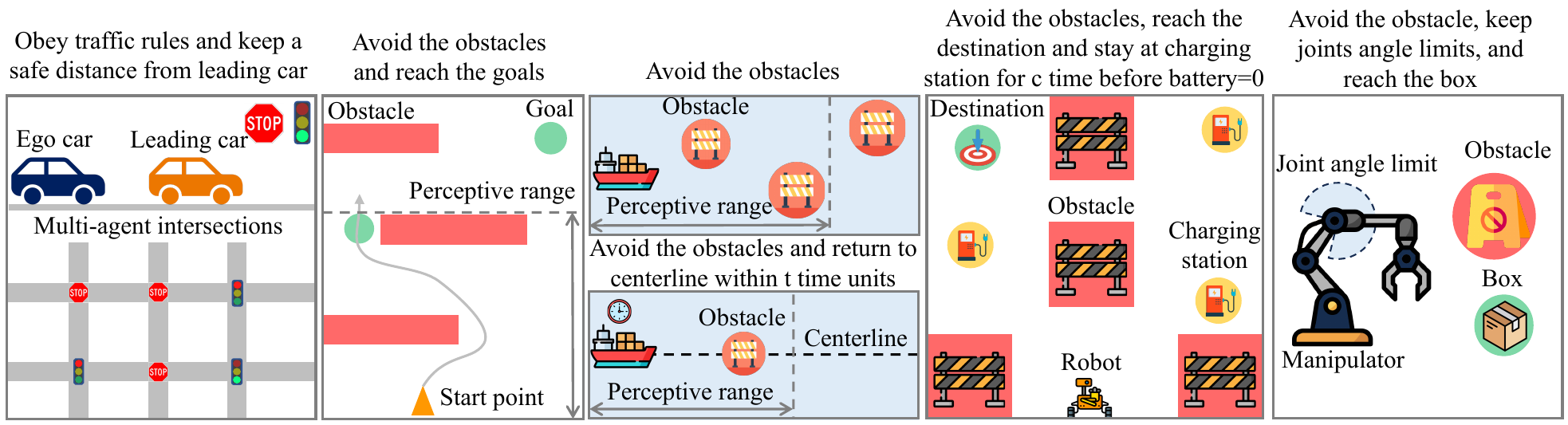}
    % \caption{Five simulation benchmarks: adaptive cruise control under traffic rules, reach-and-avoid game, ship safe control, ship safe tracking control, and robot navigation. }
    \caption{Benchmarks: learning traffic rules, reach-and-avoid game, ship safe / tracking control, navigation and manipulation. }
    \label{fig:screenshot}
\end{figure*}

Our contributions are: (1) \rbt{we are the first to use NN controllers to predict trajectories to satisfy STL in a self-supervised manner (without demonstrations)} (2) we propose a backup policy to ensure the safety of the robot when the learned policy fails to generate STL-satisfied trajectories (3) we conduct challenging experiments with complicated dynamics and STL constraints and outperform other baselines.

% This template provides authors with most of the formatting specifications needed for preparing electronic versions of their papers. All standard paper components have been specified for three reasons: (1) ease of use when formatting individual papers, (2) automatic compliance to electronic requirements that facilitate the concurrent or later production of electronic products, and (3) conformity of style throughout a conference proceedings. Margins, column widths, line spacing, and type styles are built-in; examples of the type styles are provided throughout this document and are identified in italic type, within parentheses, following the example. Some components, such as multi-leveled equations, graphics, and tables are not prescribed, although the various table text styles are provided. The formatter will need to create these components, incorporating the applicable criteria that follow.

\section{Related Work}

Temporal logic (LTL~\cite{pnueli1977temporal}, MTL~\cite{koymans1990specifying}, and STL~\cite{maler2004monitoring}) can express rich and complex robot behaviors and hence is useful for controller verification and synthesis. Plenty of motion planning algorithms have been aiming to fulfill STL specifications. We refer the readers to this survey~\cite{plaku2016motion}. 

Abstraction-based methods~\cite{tabuada2003model,fainekos2009temporal,mcmahon2014sampling} emerged the earliest, where an automaton or a graph is constructed and the search-based planning is conducted on this discrete form of abstraction. These methods often require domain knowledge and are challenging to construct automatically. 

The optimization-based approach came out for specific dynamics (e.g., linear), where the STL specifications are modeled as linear constraints and the control policy is found via Convex Quadratic Programming~\cite{lindemann2018control} or Mixed Integer Programming (MIP) ~\cite{yang2020continuous}. Though they do not need discretization or domain expertise, the computation cost is still too high and they cannot solve complicated systems. As a remedy, ~\cite{sun2022multi} plans feasible line segments and then uses tracking controllers to adapt to complex dynamics, and ~\cite{zhang2023modularized} uses separation principles to boost the MILP-solving process.

To better handle complex dynamics, sampling-based approaches are proposed which are computationally efficient and suitable for real-time applications. Representative works include STyLuS*~\cite{kantaros2020stylus} that uses biased sampling and its concurrent works~\cite{vasile2017sampling,vasile2020reactive,karlsson2020sampling} which use RRT or RRT*. Another line of work to cope with nonlinear dynamics is gradient-based methods. A differentiable measure for MTL satisfaction is developed in~\cite{pant2017smooth}. Inspired by it, STL-cg~\cite{leung2019backpropagation} handles backpropagation for (parametrized) STL formulas under machine learning frameworks. \rbt{The work}~\cite{pantazides2022satellite} learns STL for complex satellite mission planning. The work~\cite{dawson2022robust} provides a counter-example guided framework to learn robust control policies. However, gradient descent is well-known to be slow and might stuck into the local minimum. Both types of methods in this section may not guarantee the optimality or completeness of the solution.

Recently, machine learning has been widely used for STL/LTL controller synthesis, where neural networks (NN) or other forms of parametrized policy are trained to satisfy STL. These works can be further categorized into model-free reinforcement learning (RL)~\cite{aksaray2016q,li2017reinforcement,balakrishnan2019structured}, model-based RL (MBRL)~\cite{cho2018learning,kapoor2020model,cohen2021model}, and imitation learning~\cite{puranic2021learning,liu2021recurrent,hashimoto2022stl2vec}. Our method is similar to MBRL, where we train an NN controller via stochastic gradient descent (SGD) to predict policy sequences to maximize the robustness score.

% \partitle{Reinforcement Learning}

% \partitle{Safe planning via control barrier certificates}

% \partitle{Planning with Signal Temporal Logic}

% 
\section{Preliminaries \rbt{and Problem Definition}}
\subsection{Controlled Hybrid System}

Consider a continuous-time hybrid system:
\begin{equation}
\begin{cases}
\begin{aligned}
    \dot{x}=&f(x,u), \quad x\in\mathcal{C} \\
    x^+=&h(x^-), \quad x\in \mathcal{D}
\end{aligned}
\end{cases}
\label{eq:system}
\end{equation}
where $x\in\mathbb{R}^n$ denotes the system state and $u\in\mathbb{R}^m$ denotes the control input. The system state here can contain both the agent and environment information. \rbt{The set} $\mathcal{C}$ is the flow set where states follow a continuous flow map and $\mathcal{D}$ is the jump set where states encounter instantaneous changes. \rbt{A} jump \rbt{captures the scenarios where} the agent updates its local observations or resets timers. Here we consider the states and controls at discrete time steps with time horizon $T$. Given an initial state $x_0$ and a control sequence \rbt{(}$u_0, u_1, ..., u_{T-1}$\rbt{)}, following the system dynamics we can generate a trajectory \rbt{(}$x_0, x_1, ..., x_T$\rbt{)}. In our context, we call this trajectory a trace (or signal) and denote it as $s$. In this paper, we aim to learn a policy that can generate traces satisfying temporal logic properties. The temporal requirements rbt{are} formally introduced below.

\subsection{Signal Temporal Logic (STL)}
An STL formula comprises predicates, logical connectives, and temporal operators~\cite{donze2013efficient}. The predicates are of the form $\mu(x)\geq 0$ where $\mu:\mathbb{R}^n\to \mathbb{R}$ is a function with the state \rbt{$x$} as the input and returns a scalar value. STL formulas are constructed in the Backus-Naur form:
\begin{equation}
    \phi::= \top \ | \ \mu\geq 0 \ | \ \neg \phi \ | \ \phi_1 \land \phi_2 \ | \ \phi_1 \text{U}_{[a,b]}\phi_2
    \label{eq:stl_formula}
\end{equation}
where $\top$ means ``true", $\neg$ means ``negation", $\land$ means ``and", $\text{U}$ means ``until'' and $[a,b]$ is the time interval from $a$ to $b$. Other operators can be written from the elementary operators above, such as ``or": $\phi_1\lor\phi_2=\neg(\neg\phi_1 \land \neg\phi_2)$, ``infer": $\phi_1\implies\phi_2=\neg\phi_1 \lor \phi_2$, ``eventually": $\lozenge_{[a,b]} \phi=\top \text{U}_{[a,b]}\phi$ and ``always": $\square_{[a,b]} \phi = \neg\lozenge_{[a,b]}\neg\phi$. %Given a signal $s$ generated from the , an STL formula $\phi$ is satisfied at time $t$ (which is denoted as $s,t\models \phi$) with the Boolean semantics defined in~\cite{maler2004monitoring}:
We denote $s,t\models \phi$ if a signal $s$ at time $t$ satisfies an STL formula $\phi$. The detailed Boolean semantics in~\cite{maler2004monitoring} is iteratively defined as:
\begin{equation}
    \begin{aligned}
        & s,t \models \top & & (\text{naturally satisfied})\\
        & s,t \models \mu \geq 0  &\Leftrightarrow &\,\,\,\, \mu(s(t))>0 \\
        & s,t \models \neg \phi  &\Leftrightarrow &\,\,\,\, s,t\not\models \phi \\
        & s,t \models \phi_1\land \phi_2  &\Leftrightarrow &\,\,\,\, s,t\models \phi_1\text{ and } s,t\models \phi_2\\
        & s,t \models \phi_1\lor\phi_2  &\Leftrightarrow &\,\,\,\, s,t\models \phi_1\text{ or } s,t\models \phi_2\\
        & s,t \models \phi_1\implies\phi_2  &\Leftrightarrow &\,\,\,\, \text{if }  s,t\models \phi_1\text{ then } s,t\models \phi_2 \\
        & s,t \models \phi_1\text{U}_{[a,b]}\phi_2  &\Leftrightarrow & \,\,\,\,\exists t' \in [t+a, t+b] \text{ s.t. } s,t'\models \phi_2 \\& & & \,\,\,\,\text{ and } \forall t'' \in [t, t']\,\, s,t''\models \phi_1  \\
        & s,t \models \lozenge_{[a,b]}\phi  &\Leftrightarrow &\,\,\,\, \exists t' \in [t+a, t+b]\,\, s,t'\models \phi \\
        & s,t \models \square_{[a,b]}\phi  &\Leftrightarrow &\,\,\,\, \forall t' \in [t+a, t+b]\,\, s,t'\models \phi \\
        % &s,t \models \top  (\text{naturally satisfied})\\
        % &s,t \models \mu \geq 0  \Leftrightarrow \,\,\,\, \mu(s(t))>0 \\
        % &s,t \models \neg \phi  \Leftrightarrow \,\,\,\, s,t\not\models \phi \\
        % &s,t \models \phi_1\land \phi_2  \Leftrightarrow \,\,\,\, s,t\models \phi_1\text{ and } s,t\models \phi_2\\
        % &s,t \models \phi_1\lor\phi_2  \Leftrightarrow \,\,\,\, s,t\models \phi_1\text{ or } s,t\models \phi_2\\
        % &s,t \models \phi_1\to\phi_2  \Leftrightarrow \,\,\,\, \text{if }  s,t\models \phi_1\text{ then } s,t\models \phi_2 \\
        % &s,t \models \phi_1\text{U}_{[a,b]}\phi_2  \Leftrightarrow \,\,\,\,\exists t' \in [t+a, t+b]\,\, s,t'\models \phi_2 \\
        % & \,\,\,\,\text{ and } s,t''\models \phi_1 \forall t'' \in [t, t'] \\
        % &s,t \models \lozenge_{[a,b]}\phi  \Leftrightarrow \,\,\,\, \exists t' \in [t+a, t+b], s,t'\models \phi \\
        % &s,t \models \square_{[a,b]}\phi  \Leftrightarrow \,\,\,\, \forall t' \in [t+a, t+b], s,t'\models \phi \\
    \end{aligned}
    \label{eq:boolean_semantics}
\end{equation}

% A quantitative semantics $\rho$ also exists to measure how well the trace satisfies the STL formula (the margin of satisfaction)~\cite{donze2010robust}. $\rho$ is called the robustness score, and the formula is satisfied if $\rho>0$ and is violated if $\rho<0$. The robustness score is calculated with the following rules:

To measure how well the trace satisfies the STL formula, ~\cite{donze2010robust} proposes a quantitative semantics called robustness score $\rho$: the STL formula is satisfied if $\rho>0$ and is violated if $\rho<0$, and a larger $\rho$ reflects a larger margin of satisfaction. The robustness score is calculated with the following rules:
\begin{equation}
    \begin{aligned}
        & \rho(s,t, \top)= 1 , \quad \rho(s,t, \mu\geq 0)= \mu(s(t))\\
        & \rho(s,t, \neg \phi)= -\rho(s,t,\phi)\\
        % % & \rho(s,t, \mu\geq 0)= \mu(s(t))\\
        % % & \rho(s,t, \neg \phi)= -\rho(s,t,\phi)\\
        & \rho(s,t, \phi_1 \land \phi_2)= \min\{\rho(s,t,\phi_1),\rho(s,t,\phi_2) \}\\
        & \rho(s,t, \phi_1 \lor \phi_2)= \max\{\rho(s,t,\phi_1),\rho(s,t,\phi_2) \}\\
        & \rho(s,t, \phi_1 \implies \phi_2)= \max\{-\rho(s,t,\phi_1),\rho(s,t,\phi_2) \}\\
        & \rho(s,t, \phi_1 \text{U}_{[a,b]} \phi_2)= \\
        & \quad \sup\limits_{t'\in[t+a,t+b]}\min\left\{\rho(s,t',\phi_2), \inf\limits_{t''\in[t,t']}\rho(s,t'',\phi_1) \right\}\\
        & \rho(s,t, \lozenge_{[a,b]} \phi)= \sup\limits_{t'\in[t+a,t+b]}\rho(s,t',\phi)\\
        & \rho(s,t, \square_{[a,b]} \phi)= \inf\limits_{t'\in[t+a,t+b]}\rho(s,t',\phi)
    \end{aligned}
    \label{eq:robustness_score}
\end{equation}

\rbt{\subsection{Problem Formulation}}
\rbt{With the definition of the dynamical system and STL specifications, the problem under consideration is as follows.}
\begin{problem}
\rbt{Given a system model in Eq.~\eqref{eq:system}, an STL formula $\phi$ in Eq.~\eqref{eq:stl_formula} and an initial state set $\mathcal{X}_0\subset \mathbb{R}^n$, find a control policy $\pi$ such that starting from any state $x\in\mathcal{X}_0$, the trajectory $s_\pi(x)$ of the resulting 
 closed-loop system satisfies the formula $\phi$, i.e., $\forall x \in \mathcal{X}_0, \, s_\pi(x)\models \phi$. }
\end{problem}

% Computing $\rho$ is not differentiable due to the min and max operators. To use gradient-based approaches in controller synthesis, we use the approximated robustness score $\hat{\rho}$ in~\cite{pant2017smooth}, which replaces the max (min) operators in Eq.~\eqref{eq:robustness_score} with the smooth max (min) operators below: 
% \begin{equation}
%     \begin{aligned}
%         &\widetilde{\max}_k(x_1,x_2,...):=\frac{1}{k}\log \left(e^{k x_1} + e^{k x_2} + \cdots \right)  \\
%         &\widetilde{\min}_k(x_1,x_2,...):=-\widetilde{\max}_k(-x_1,-x_2,...)  \\
%     \end{aligned}
%     \label{eq:approximated_robustness_score}
% \end{equation}
% where $k$ is a scaling factor for this approximation.If $k\to \infty$, the approximation is close to the max/min operators. With $\hat{\rho}$, we are ready to set up our controller learning framework.

\section{Methodology}
\subsection{STL Satisfaction as An Optimization Problem}
To satisfy $\phi$, we measure the satisfaction rate of $\phi$ using the robustness score and thus form the boolean STL satisfaction task as an optimization problem. Denote the policy parametrized with $\theta$ as $\pi_\theta$. For $x\in\mathcal{X}_0$, the policy predicts a sequence of controls: $\pi_{\theta}(x)=u_{0:T-1}$. Following system dynamics, we can generate the trajectory $s_{\pi_{\theta}}(x)$ with horizon $T+1$. Then we compute the robustness score $\rho(s_{\pi_{\theta}}(x), \phi)$ for the STL formula $\phi$ on the trajectory $s_{\pi_{\theta}}(x)$ at time 0 (we omit $t$ for brevity). Our goal is: \rbt{(omit dynamic constraints)}
\begin{equation}
    \text{Find } \pi_{\theta}, \text{s.t. } \rho(s_{\pi_{\theta}(x)},\phi)>0, \forall x\in \mathcal{X}_0
\end{equation}
\rbt{Assuming $x$ follows uniform distribution on $\mathcal{X}_0$, we aim to find the optimal policy which maximizes the expected truncation robustness score:}
\begin{equation}
    \begin{aligned}
        \pi^*_{\theta} = \mathop{\arg\max}\limits_{\pi_{\theta}} \mathop{\mathbb{E}}\limits_{x\sim \mathcal{X}_0}\left[ \rbt{\min}\left\{\rho(s_{\pi_{\theta}(x)},\phi), \gamma\right\}\right]
    \end{aligned}
    \label{eq:optimization}
\end{equation}
where $\gamma>0$ is the \rbt{truncation} factor. \rbt{The $\min$ operator in the expectation encourages the policy to improve ``hard" trajectories (with robustness scores $<\gamma$) rather than further increasing ``easy" trajectories that already achieve high robustness scores ($\geq\gamma$). This helps achieve high robustness score for all possible cases}. \rbt{Another advantage is that if the optimal value is $\gamma$, the STL $\phi$ is guaranteed to be satisfied on all sampled initial states.} In the experiments, we set $\gamma=0.5$. An ablation study on $\gamma$ is in Table~\ref{tab:abl_nn}.

\subsection{Neural Network Controller Learning for STL Satisfication}
 \rbt{We aim to solve Eq.~\eqref{eq:optimization} using neural networks. Unlike~\cite{dawson2022robust}, which solves the STL satisfaction online, we learn the control policy in training, and thus, our approach can run in real-time in testing. 
 %Different from~\cite{liu2021recurre\nt}, we don't require nominal controllers or Control Barrier Functions to ensure safety.
We use a fully-connected network (MLP) $\pi_{\theta}$ to represent the control policy. At each training step, we first sample initial states from $\mathcal{X}_0$ and use $\pi_{\theta}$ to predict a sequence of actions. Then we roll-out trajectories based on these actions and calculate the robustness score to form a loss function shown in Eq.~\eqref{eq:optimization}. Finally, we update the parameters of the neural network controller guided by the loss function via stochastic gradient descent.}
 However, there remain two challenges for us to apply the gradient method. First, the hybrid systems in Eq.~\ref{eq:system} are non-differentiable at the mode-switching instant. Secondly, $\rho$ is not differentiable due to the max (min) and sup (inf) operators in Eq.~\eqref{eq:robustness_score}.

 % To tackle the non-smoothness in the hybrid systems dynamics, we assume a membership function $I_{C}:\mathcal{X}\to \mathbb{R}$ exists to imply whether a state $x_t$ is in the flow set ($I_{C}(x_t)>0$) or the jump set ($I_{C}(x_t)<0$). Thus the forward dynamics can be written as $x_{t+1}=\mathbbm{1}\{I_{C}(x_t)>0\} f(x_t,u_t)\Delta t + (1-\mathbbm{1}\{I_{C}(x_t)>0\}) h(x_t)$. Next, we approximate the $\mathbbm{1}(I_{C}(x_t)>0)$ using $\widetilde{I}_{w}(x_t)=(1+\text{Tanh}(w\cdot I_C(x_t)))/2$ with a scaling factor $w>0$  to control the approximation ($\lim\limits_{w\to \infty}\widetilde{I}_{w}(x_t)=\mathbbm{1}(I_{C}(x_t)>0)$). The dynamics now is: %$x_{t+1}=\widetilde{I}_{w}(x_t) f(x_t,u_t)\Delta t + (1-\widetilde{I}_{w}(x_t)) h(x_t)$
  To tackle the non-smoothness in the hybrid systems dynamics, we assume a membership function $I_{C}:\mathcal{X}\to \mathbb{R}$ exists to imply whether a state $x_t$ is in the flow set ($I_{C}(x_t)>0$) or the jump set ($I_{C}(x_t)<0$). Thus the forward dynamics can be written as $x_{t+1}=\mathds{1}\{I_{C}(x_t)>0\} f(x_t,u_t)\Delta t + (1-\mathds{1}\{I_{C}(x_t)>0\}) h(x_t)$. Next, we approximate the $\mathds{1}(I_{C}(x_t)>0)$ using $\widetilde{I}_{w}(x_t)=(1+\text{Tanh}(w\cdot I_C(x_t)))/2$ with a scaling factor $w>0$  to control the approximation ($\lim\limits_{w\to \infty}\widetilde{I}_{w}(x_t)=\mathds{1}(I_{C}(x_t)>0)$). The dynamics now is: %$x_{t+1}=\widetilde{I}_{w}(x_t) f(x_t,u_t)\Delta t + (1-\widetilde{I}_{w}(x_t)) h(x_t)$
\begin{equation}
    x_{t+1}=\widetilde{I}_{w}(x_t) f(x_t,u_t)\Delta t + (1-\widetilde{I}_{w}(x_t)) h(x_t)
    \label{eq:approximated_hybrid_dynamics}
\end{equation}
%  To tackle the non-smoothness in the hybrid systems dynamics, we use a membership function $I_{C}:\mathcal{X}\to \{0,1\}$ to imply whether a state $x_t$ is in the flow set ($I_{C}(x_t)=1$) or the jump set ($I_{C}(x_t)=0$). Thus the forward dynamics can be written as $x_{t+1}=I_{C}(x_t)f(x_t,u_t)\Delta t + (1-I_{C}(x_t)) h(x_t)$. Next, we approximate the $I_{C}(x_t)$ using $\widetilde{I}_{C,w}(x_t)=(1+\text{Tanh}(w\cdot I_C(x_t)))/2$ with $w>0$ a scaling factor to control the approximation degree ($\lim\limits_{w\to \infty}\widetilde{I}_{C,w}(x_t)=\mathbbm{1}(I_{C}(x_t)>0)$). We approximate the hybrid forward dynamics:
% \begin{equation}
%     x_{t+1}=\widetilde{I}_{C,w}(x_t) f(x_t,u_t)\Delta t + (1-\widetilde{I}_{C,w}(x_t)) h(x_t)
%     \label{eq:approximated_hybrid_dynamics}
% \end{equation}
and we can learn STL satisfaction from this hybrid system.

 To backpropagate the gradient through STL, we use the approximated robustness score $\tilde{\rho}$~\cite{pant2017smooth}, which replaces the max (min) and sup (inf) operators with smooth max (min): 
\begin{equation}
    \begin{aligned}
        &\widetilde{\max}_k(x_1,x_2,...):=\frac{1}{k}\log \left(e^{k x_1} + e^{k x_2} + \cdots \right)  \\
        &\widetilde{\min}_k(x_1,x_2,...):=-\widetilde{\max}_k(-x_1,-x_2,...)  \\
    \end{aligned}
    \label{eq:approximated_robustness_score}
\end{equation}
% $\widetilde{\max}_k(x_1,x_2,...):=\frac{1}{k}\log \left(e^{k x_1} + e^{k x_2} + \cdots \right), \, \widetilde{\min}_k(x_1,x_2,...):=-\widetilde{\max}_k(-x_1,-x_2,...)$
where $k$ is a scaling factor for this approximation. \rbt{If $k\to \infty$, the operator $\widetilde{\max}=\max$ and similarly $\widetilde{\min}=\min$}. We use $k=500$ in our training. An ablation study on $k$ is in Table~\ref{tab:abl_nn}. 

Now the framework is differentiable for both the STL robustness score calculation and the system dynamics, we encode the objective in Eq.~\eqref{eq:optimization} using the loss function: 
\begin{equation}
    \mathcal{L}_{\text{STL}}=\frac{1}{|\mathcal{D}_0|}\sum\limits_{x\in \mathcal{D}_0}\max(0, \gamma-\tilde{\rho}(s_{\pi_{\theta}(x)},\phi))
    \label{eq:loss_stl}
\end{equation}
where a finite number of states $x$ are uniformly sampled from $\mathcal{X}_0$ to form the training set $\mathcal{D}_0$. \rbt{Aside from constraint satisfaction, the agent might also need to maximize some performance indices (e.g., ``reach the destination as fast as possible.")} We hence form the following loss function:
\begin{equation}
    \mathcal{L}_{Total}=\mathcal{L}_{Perf}+\lambda \mathcal{L}_{\text{STL}}
    \label{eq:loss_total}
\end{equation}
\rbt{where $\lambda>0$ weighs the performance objective $\mathcal{L}_{Perf}$ and the STL violations $\mathcal{L}_{STL}$. \rbt{We admit the  $\lambda \mathcal{L}_{\text{STL}}$ term cannot guarantee the learned policy always satisfy the STL requirement, but empirically we found out this can bring high STL satisfaction rate without much effort in hyperparameter tuning.} The $\mathcal{L}_{Perf}$ is often the Euclidean distance between state $x_t$ (in the trajectories starting from $x\sim\mathcal{D}_0$) and goal state $x^*$, i.e., $\mathcal{L}_{Perf}=\frac{1}{|\mathcal{D}_0 |(T+1)}\sum\limits_{x\sim \mathcal{D}_0}|x_t-x^*|$.}

\subsection{MPC-based Deployment with a Backup Policy}
In testing, we follow an online MPC manner. At each time step $t$, the controller $\pi_{\theta}$ receives the state $x_t$ and predicts a sequence of commands $u_{0:T-1}$, then we choose the first command $u_0$ for the agent to execute. %This process repeats until the task is finished or exceeds the time horizon.
Ideally, this policy will satisfy the STL constraints. 
However, this might not hold in testing due to: (1) imperfect training and (2) out-of-distribution scenario. Thus, we propose a backup policy. We monitor whether the predicted trajectory satisfies the STL specification. If a violation occurs, we sample $M$ trajectories in length $T_0+1$ with $T_0<T$. For the i-th trajectory $\tilde{\xi}_i=\{x_0^i, x_1^i,...x_{T_0}^i\}$, we evaluate our controller at the final state $x_{T_0}^i$ and rollout a trajectory $\hat{\xi}_i=\{x_{T_0+1}^i,x_{T_0+2}^i...,x_{T}^i\}$ (we only keep the first $T-T_0$ timesteps). Since  Neural Networks allow batch operations, all the sampled trajectories $\{\tilde{\xi}_i\}_{i=1}^M$ can be efficiently processed in one forward step to get $\{\hat{\xi}_i\}_{i=1}^M$. Finally, we select the trajectory $\xi_i=(\tilde{\xi}_i,\hat{\xi}_i)=\{x_0^i, ..., x_{T_0}^i, x_{T_0+1}^i, ...x_T^i\}$ with the highest robustness score and pick its first action to execute. 

If this trajectory still cannot satisfy the STL specification, we choose a trajectory that only satisfies the safety condition:
\begin{equation}
\mathop{\text{argmax}}\limits_i \rho(\xi_i,t,\phi), \text{ s.t. } \xi_i, t |= \phi_\mathtt{safe}
\label{eq:backup-2}
\end{equation} 
where $\phi_\mathtt{safe}$ contains all the safety constraints in $\phi$. We start with $T_0=1$, obtain $\xi_i$ using the above method, and gradually increase $T_0$ until a solution for Eq.~\eqref{eq:backup-2} is found. If there is still no feasible solution after $T_0=T$, we choose $\xi_i$ with the longest sub-trajectory starting from $x_0^i$ that satisfies $\phi_\mathtt{safe}$ and pick the first action to execute. This ensures at least $\phi_\mathtt{safe}$ are satisfied and the agent can recover to satisfy $\phi$ in the earliest time. %With a large sampling size of $M$, the backup policy has a large probability of guaranteeing STL/safety. In experiments, we discretize the action space to $L$ bins, and hence the sampling size is $M=L^{T_0}$.
we discretize the action space to $L$ bins, and hence the sampling size is $M=L^{T_0}$. We prove the backup policy has a probabilistic guarantee to find a feasible solution.

% In some extreme cases, even   In the first step, we still recover the STL, by first sampling and then recovering. In the second step, if the safety is violated, we propose another policy that can guarantee at least the safety constraint.  As a step back, we try to achieve only the safety conditions, using a brute-force search method, which will add up to time complexity but ensure the safety.

\begin{theorem}
%Denote the solution space $\mathcal{U}=\prod\limits_{i=1}^{T_0} \prod\limits_{j=1}^m [u^{min}_j, u^{max}_j]$. 
\rbt{Assume that a policy $u^*$ exists} with a $\delta$-radius neighborhood satisfying constraints $\phi$, i.e., 
%\exists u^*_{0:T_0-1} \text{ s.t. } \max\limits_{t,i}|\tilde{u}_{t,i}-u^*_{t,i}|_1 \leq \delta \implies \tilde{u} \models\phi$, 
$u \models\phi,\, \forall u \in \{u|\max\limits_{t,i}|u_{t,i}-u^*_{t,i}|_1 \leq \delta \}$, and each step's policy $u^*_t$ is uniformly distributed in $\prod\limits_k[u^\mathtt{min}_i, u^\mathtt{max}_i]$.  The probability our algorithm finds a solution is: $\min\{1, (\frac{((2L-4)\delta)^{mT}}{\prod\limits_{i=1}^m (u^\mathtt{max}_i-u^\mathtt{min}_i)^{T}}\}$.
\end{theorem}
\begin{proof}
Our backup controller samples from a grid in the policy space at each time step. The probability that a  solution can be found is equal to that the $\delta$-region contains a grid point at each time step, which is greater than or equal to the probability that the union of the $\delta$-hypercube centered at all grids covers the policy $u^*$.  The volume of the policy space at each time step is $\prod\limits_{i=1}^m (u^\mathtt{max}_{i}-u^\mathtt{min}_{i})$. Each hypercube has a side length $2\delta$. Thus the volume of the union of the hypercubes is greater than (as we omit the cubes that are at the boundary of the policy space) $(L-2)^m (2\delta)^m$. Thus the probability of the union covering $u^*_t$ is $\min\{1, \frac{((2L-4)\delta)^{m}}{\prod\limits_{i=1}^m (u^\mathtt{max}_i-u^\mathtt{min}_i)}\}$, and for $T$ steps, the probability is powered by $T$ to derive the expected result shown in Theorem 1.
\end{proof}

\rbt{\subsection{Remarks on the STL constraints}}
\rbt{To handle different configurations, we augment the system states with additional parameters such as obstacle radius, locations and time constraints. These parameters stay constant unless encountering a reset. The policy learned from this augmentation can solve a family of STL formulas and can adapt to unseen configurations without further fine-tuning.}

\rbt{A wide range of requirements commonly used in robot tasks can be represented by STL, owing to the flexible form of atomic propositions (AP). Denote the 2D location of a robot as $x_\mathtt{[0:2]}\in \mathbb{R}^2$. We use $\mu(x)=r-||x_\mathtt{[0:2]}-x_\mathtt{obs}||_2$ to check collision with a round obstacle at $x_\mathtt{obs}\in \mathbb{R}^2$ with radius $r\in\mathbb{R}$, where $||\cdot||_2$ represents the Euclidean norm. For a polygon region $S=\{y:Ay\leq b\}$ with $A\in\mathbb{R}^{k\times 2}$ and $b\in\mathbb{R}^k$, $\mu(x)=b-Ax_\mathtt{[0:2]}$ can check whether the robot is in $S$. We use $\mu(x)=-(x_\mathtt{[0]}-x_\mathtt{min})(x_\mathtt{[0]}-x_\mathtt{max})$ to check whether the agent's $x$-coordinate is in the interval $[x_\mathtt{min}, x_\mathtt{max}]$. Furthermore, we can evaluate whether a system mode or a scenario has been activated by checking an indicator $I_1\in\{0,1\}$ via AP: $\mu(I_1)=I_1-0.5$. }

\rbt{However, it is difficult to cope with constraints with a global time interval. Consider $\phi=\square_{[5,10]}\text{Stay}(A)$ which means \textit{``Always stay at A in the time interval [5,10]"}. One step later, the agent will still try to satisfy $\square_{[5,10]}\text{Stay}(A)$, which actually should be updated to $\square_{[4,9]}\text{Stay}(A)$. Thus we need our policy to be aware of the STL time interval update process during training. To tackle this, we augment the system state with the timer variables and transform the original STL formula to extra STL constraints on those timer variables. To be more specific, the timer variables follow the dynamic: $\tau_{t+1} = \tau_{t} + \Delta t$ and the extra STL constraint needed to satisfy is $\square_{[0,T]} \left((5\leq\tau\leq 10)\to \text{Stay}(A)\right)$. We provide more examples and details in the experiment section \ref{sec:exp}.
}

\section{Experiments}
\label{sec:exp}

We conduct experiments with diverse dynamics and task specifications. Our method's training time is only 0.5X the time needed for training RL. In testing, our method is on par with the best baseline on simple benchmarks and achieves the highest STL accuracy on more complicated cases. As for runtime, our approach (without backup policy) is 10X-100X faster than classic planning methods such as MPC and MILP.

\subsection{Experiment setups}

\partitle{Baselines} We compare with RL, \rbt{model-based RL (MBRL),} and classical approaches. For RL, we train Soft Actor Critic~\cite{haarnoja2018soft}~\cite{raffin2021stable} under \rbt{five} random seeds with varied rewards. \textbf{RL\textsubscript{R}}: uses a hand-crafted reward. \textbf{RL\textsubscript{S}}: uses STL robustness score as the reward. \textbf{RL\textsubscript{A}}: uses STL accuracy as the reward. \rbt{The MBRL baselines are \textbf{MBPO}:~\cite{janner2019trust} with STL accuracy as the reward, 
\textbf{PETS}:~\cite{chua2018deep} with a hand-crafted reward, and \textbf{CEM}: Cross Entropy Method~\cite{de2005tutorial} with STL robustness reward.} \rbt{The rests are} \textbf{MPC}: Model predictive control via Casadi~\cite{andersson2019casadi} for nonlinear systems and Gurobi\rbt{~\cite{gurobi2021gurobi}} for linear dynamics, \textbf{STL\textsubscript{M}}: An official implementation for STL-MiLP~\cite{sun2022multi} with a PD control for nonlinear dynamics if needed, and \textbf{STL\textsubscript{G}}: A gradient-based method similar to~\cite{dawson2022robust}. 

\partitle{Implementation details} For our method, we set $\gamma=0.5$, \rbt{$k=500$}, $T\in[10,25]$ and $\Delta t\in[0.1\mathtt{s}, 0.2\mathtt{s}]$. Our controller is a three-layer MLP with 256 hidden units in each layer. We uniformly sample 50000 points and train for 50k steps (250k for navigation). We update the controller 
% and use the NN controller to rollout trajectories for $10\sim25$ timesteps with $\Delta t\in[0.1s,0.2s]$. 
via the Adam optimizer~\cite{kingma2014adam} with a learning rate $3\times 10^{-4}$ for most tasks. Training in PyTorch~\cite{paszke2019pytorch} takes 2-12 hours on a V100 GPU.

\partitle{Metrics} \rbt{In testing we evaluate the average STL accuracy (the ratio of the short segments starting at each step satisfies the STL)} and the computation time. For RL baselines (\textbf{RL\textsubscript{R}}, \textbf{RL\textsubscript{S}}, \textbf{RL\textsubscript{A}}), we also compare the STL accuracy in training.

% \begin{figure*}[htbp]
% \centering
%     \begin{subfigure}[t]{0.155\textwidth}
%         \includegraphics[width=\textwidth]{figs/rewards_Exp1-Traffic.png}
%         \caption{Traffic}
%     \end{subfigure}
%     \begin{subfigure}[t]{0.155\textwidth}
%         \includegraphics[width=\textwidth]{figs/rewards_Exp2-Maze game.png}
%         \caption{Reach-n-avoid}
%     \end{subfigure}
%     \begin{subfigure}[t]{0.155\textwidth}
%         \includegraphics[width=\textwidth]{figs/rewards_Exp3-Ship1.png}
%         \caption{Ship-safe}
%     \end{subfigure}
%     \begin{subfigure}[t]{0.155\textwidth}
%         \includegraphics[width=\textwidth]{figs/rewards_Exp4-Ship2.png}
%         \caption{Ship-track}
%     \end{subfigure}
%     \begin{subfigure}[t]{0.155\textwidth}
%         \includegraphics[width=\textwidth]{figs/rewards_Exp5-Rover.png}
%         \caption{Robot}
%     \end{subfigure}
%     \begin{subfigure}[t]{0.155\textwidth}
%         \includegraphics[width=\textwidth]{example-image-c}
%         \caption{Arm}
%     \end{subfigure}
%     \caption{STL accuracy during training}
%     \label{fig:training-curve}
% \end{figure*}

\begin{figure*}[htbp]
\centering
    \includegraphics[width=\textwidth]{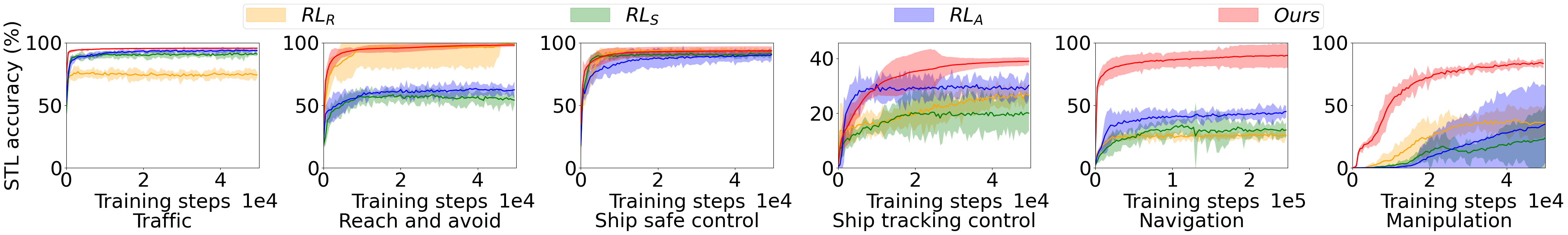}
    \caption{STL accuracy during training}
    \label{fig:training-curve}
\end{figure*}

% \begin{figure*}[htbp]
% \centering
%     \begin{subfigure}[t]{0.155\textwidth}
%         \includegraphics[width=\textwidth]{figs/e1_car_acc.png}
%         \caption{Traffic}
%     \end{subfigure}
%     \begin{subfigure}[t]{0.155\textwidth}
%         \includegraphics[width=\textwidth]{figs/e2_game_acc.png}
%         \caption{Reach-n-avoid}
%     \end{subfigure}
%     \begin{subfigure}[t]{0.155\textwidth}
%         \includegraphics[width=\textwidth]{figs/e3_ship_safe_acc.png}
%         \caption{Ship-safe}
%     \end{subfigure}
%     \begin{subfigure}[t]{0.155\textwidth}
%         \includegraphics[width=\textwidth]{figs/e4_ship_track_acc.png}
%         \caption{Ship-track}
%     \end{subfigure}
%     \begin{subfigure}[t]{0.155\textwidth}
%         \includegraphics[width=\textwidth]{figs/e5_rover_acc.png}
%         \caption{Robot}
%     \end{subfigure}
%     \begin{subfigure}[t]{0.155\textwidth}
%         \includegraphics[width=\textwidth]{example-image-b}
%         \caption{Arm}
%     \end{subfigure}
%     \caption{STL accuracy at test phase}
%     \label{fig:testing-acc}
% \end{figure*}

\begin{figure*}[htbp]
\centering
\includegraphics[width=\textwidth]{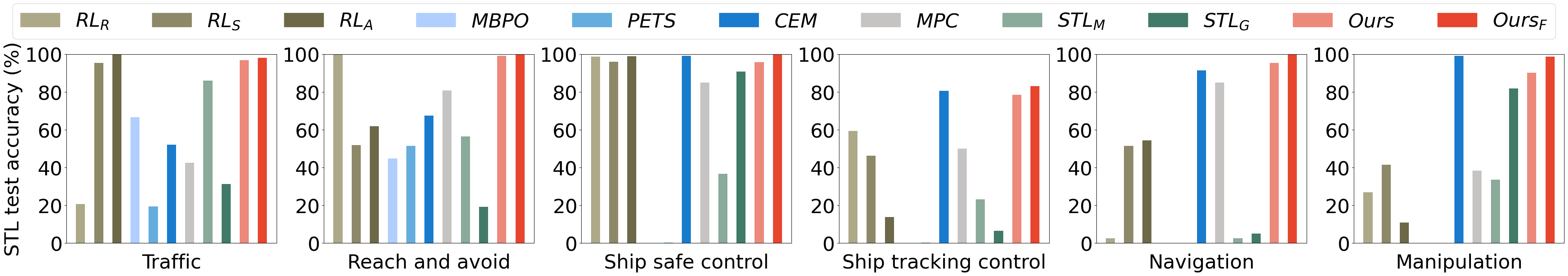}
\caption{STL accuracy at test phase}
\label{fig:testing-acc}
\end{figure*}

% \begin{figure*}[htbp]
% \centering
%     \begin{subfigure}[t]{0.155\textwidth}
%         \includegraphics[width=\textwidth]{figs/e1_car_t.png}
%         \caption{Traffic}
%     \end{subfigure}
%     \begin{subfigure}[t]{0.155\textwidth}
%         \includegraphics[width=\textwidth]{figs/e2_game_t.png}
%         \caption{Reach-n-avoid}
%     \end{subfigure}
%     \begin{subfigure}[t]{0.155\textwidth}
%         \includegraphics[width=\textwidth]{figs/e3_ship_safe_t.png}
%         \caption{Ship-safe}
%     \end{subfigure}
%     \begin{subfigure}[t]{0.155\textwidth}
%         \includegraphics[width=\textwidth]{figs/e4_ship_track_t.png}
%         \caption{Ship-track}
%     \end{subfigure}
%     \begin{subfigure}[t]{0.155\textwidth}
%         \includegraphics[width=\textwidth]{figs/e5_rover_t.png}
%         \caption{Robot}
%     \end{subfigure}
%     \begin{subfigure}[t]{0.155\textwidth}
%         \includegraphics[width=\textwidth]{example-image-a}
%         \caption{Arm}
%     \end{subfigure}
%     \caption{Computation time at test phase}
%     \label{fig:testing-time}
% \end{figure*}
\begin{figure*}[htbp]
\centering
\includegraphics[width=\textwidth]{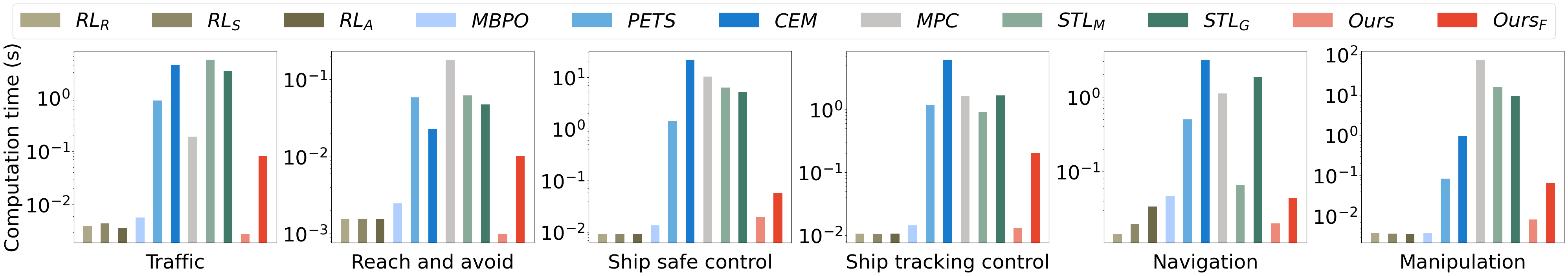}
\caption{Computation time at test phase}
\label{fig:testing-time}
\end{figure*}

% \begin{figure*}[h]
% \centering
%     \begin{subfigure}[t]{0.19\textwidth}
%         \includegraphics[width=\textwidth]{figs/e3_ship_safe_safety.png}
%         \caption{Ship-safe}
%     \end{subfigure}
%     \begin{subfigure}[t]{0.19\textwidth}
%         \includegraphics[width=\textwidth]{figs/e4_ship_track_safety.png}
%         \caption{Ship-track, Safety}
%     \end{subfigure}
%     \begin{subfigure}[t]{0.19\textwidth}
%         \includegraphics[width=\textwidth]{figs/e5_rover_safety.png}
%         \caption{Mars-Rover, Safety}
%     \end{subfigure}
%     \begin{subfigure}[t]{0.19\textwidth}
%         \includegraphics[width=\textwidth]{figs/e5_rover_battery.png}
%         \caption{Mars-Rover, Battery}
%     \end{subfigure}
%     \begin{subfigure}[t]{0.19\textwidth}
%         \includegraphics[width=\textwidth]{figs/e5_rover_goals.png}
%         \caption{Mars-Rover, Goals}
%     \end{subfigure}
    
%     \caption{Other metrics}
% \end{figure*}

% \begin{figure*}[h]
% \centering
%     \begin{subfigure}[t]{0.75\textwidth}
%         \includegraphics[width=\textwidth]{figs/rewards.png}
%         \caption{}
%     \end{subfigure}
%     \caption{(TODO) Reward curves}
% \end{figure*}

% \begin{figure*}[h]
% \centering
%     \begin{subfigure}[t]{0.75\textwidth}
%         \includegraphics[width=\textwidth]{figs/impl_status.png}
%         \caption{Current status}
%     \end{subfigure}
%     \caption{Current status}
% \end{figure*}

\subsection{Benchmarks}
\subsubsection{Driving with traffic rules}
\rbt{We consider driving near intersections where routes and lateral control are provided. The state} $(x, v, I_\mathtt{light}, \tau, \Delta x, v_\mathtt{lead}, I_\mathtt{yield})^T$ is for ego car offset, velocity, light indicator ($0$ for stop sign and $1$ for traffic light), timer (stopped time or traffic light phase), leading vehicle distance, its speed, and yield signal ($1$ for yield and $0$ for not). The (partial) dynamics are:
% \begin{equation}
%     \dot{\begin{pmatrix} x \\ v \\ I_{light} \\ \tau \\ \Delta x \\ v_{lead} \\ I_{yield} \end{pmatrix}} = \begin{pmatrix}
%         v \\ u \\ 0 \\ (1-I_{light})\mathbbm{1}(\text{at stop sign})+I_{light} \\ v_{lead} - v \\ 0 \\ 0
%     \end{pmatrix}
% \end{equation}
% \begin{equation}
%     \dot{\begin{pmatrix} x \\ v \\ I_\mathtt{light} \\ \tau \\ \Delta x \\ v_\mathtt{lead} \\ I_\mathtt{yield} \end{pmatrix}} = \begin{pmatrix}
%         v \\ u \\ 0 \\ (1-I_\mathtt{light})\mathbbm{1}(\text{at stop sign})+I_\mathtt{light} \\ v_\mathtt{lead} - v \\ 0 \\ 0
%     \end{pmatrix}
% \end{equation}
% \begin{equation}
%     \begin{aligned}
%         \dot{x}=v, \quad\dot{v}=u, \quad \dot{(\Delta x)} = v_\mathtt{lead}-v\\
%         \dot{\tau}=(1-I_\mathtt{light})\mathds{1}(\text{at stop sign})+I_\mathtt{light}
%     \end{aligned}
% \end{equation}
$      \dot{x}=v, \,\dot{v}=u, \, \dot{(\Delta x)} = v_\mathtt{lead}-v,\,
        \dot{\tau}=(1-I_\mathtt{light})\mathds{1}(\text{at stop sign})+I_\mathtt{light}$
where $\mathds{1}(\text{at stop sign})=\mathds{1}\left(x(x+1)\leq 0\right)$ as $x=0$ means being at the intersection, and $u$ is the control. $x, I_\mathtt{light}, \tau$ will reset at a new intersection, $\Delta x, v_\mathtt{lead}$ will reset when the leading car changes, and $I_\mathtt{yield}$ will change when the external yield command is emitted. The rules are: (1) never collide with the leading car, (2) stop by the stop sign for 1 second and then enter the intersection if no yield (3) \rbt{stop by the intersection if it is red light}. Thus,
% \begin{equation}
%     \Phi=(\neg I_{light}\implies \phi_1) \land (I_{light}\implies  \phi_2) \land \phi_3 
% \end{equation}
$
    \Phi=(\neg I_{\mathtt{light}}\implies \phi_1) \land (I_{\mathtt{light}}\implies  \phi_2) \land \phi_3 
$
where ($T_\mathtt{total}=T_r+T_g$):
\begin{equation}
    \begin{aligned}
    & \phi_1 = \lozenge_{[0,T]}(\tau>1) \land \left(I_\mathtt{yield}\implies \square_{[0,T]}(x<0)\right)\\
    & \phi_2 = \square_{[0,T]}(\tau\%(T_\mathtt{total})>T_{r} \lor x(x-x_\mathtt{inter})>0) \\
    & \phi_3 = \square_{[0,T]}(\Delta x>0)\\
    \end{aligned}
\end{equation}
where $T_{r}$ and $T_{g}$ are red and green light phase time, \% is modulo operator and $x_\mathtt{inter}$ is the intersection width. We train $\pi_\theta$ under each sampled scenario (traffic light, stop sign, yield, leading vehicle) at one intersection. In testing, we conduct multi-agent planning with 20 cars and 20 junctions. 

\subsubsection{Reach-n-avoid game}
\rbt{An agent navigates to reach goals and avoid obstacles in a 2D maze shown in Fig.~\ref{fig:screenshot}(b). It can see up to two levels.} The agent state is $(x, v, \Delta y, x_0, l_0, g_0, x_1, l_1, g_1)^T$ where $x$ and $v$ denote the agent's horizontal position and velocity, $\Delta y$ is the vertical distance to the nearest level above. Here $x_0$ is the obstacle's leftmost horizontal position, $l_0$ is the obstacle width, and $g_0$ is the goal location relative to the obstacle (-1 if no goal). $x_1, l_1, g_1$ are for the second level above. \rbt{The agent dynamics are}:
$\dot{x}=v, \dot{v}=a, \dot{(\Delta y)}=c$
where $a$ is the control and $c$ is a constant. Here $\Delta y, \Delta x_0, l_0, g_0, \Delta x_1, l_1, g_1$ will reset once the agent passes a level. The STL is %\Phi=\phi_1 \land \phi_2 \land \phi_3 \land \phi_4$, 
% \begin{equation}
%     \Phi=\phi_1 \land \phi_2 \land \phi_3 \land \phi_4 
% \end{equation}
$\Phi=\phi_1 \land \phi_2 \land \phi_3 \land \phi_4$:
% \begin{equation}
%     \begin{aligned}
%         & \phi_1 = g_0>0 \implies \lozenge_{[0,T]} \left((x-l_0)^2+\Delta y^2<r^2\right) \\
%         & \phi_2 = \square_{[0,T]} (\Delta y(\Delta y-h)<0 
%         \\
%         & \quad\quad \quad \quad\implies (x- x_0)(x-\ x_0-l_0)>0) \\
%         & \phi_3 = g_1>0 \implies \lozenge_{[0,T]} \left((x-l_1)^2+(\Delta y-d)^2<r^2\right) \\
%         & \phi_4 = \square_{[0,T]} ((\Delta y-d)(\Delta y-h-d)<0 \\
%         & \quad\quad \quad \quad \implies (x-x_1)(x-x_1-l_1)>0) \\
%     \end{aligned}
% \end{equation}
\begin{equation}
    \begin{aligned}
        & \phi_1 = g_0>0 \implies \lozenge_{[0,T]} \left((x-l_0)^2+\Delta y^2<r^2\right) \\
        & \phi_2 = \square_{[0,T]} (\Delta y(\Delta y-h)<0 
        \implies \Delta x_0(\Delta x_0-l_0)>0) \\
        & \phi_3 = g_1>0 \implies \lozenge_{[0,T]} \left((x-l_1)^2+(\Delta y-d)^2<r^2\right) \\
        & \phi_4 = \square_{[0,T]} (\Delta y_1(\Delta y_1-h)<0 
        \implies \Delta x_1(\Delta x_1-l_1)>0) \\
    \end{aligned}
\end{equation}
where $r$ is the goal radius, $h$ is the obstacle's height, $\Delta x_i=x-x_i$, $\Delta y_1=\Delta y-d$ and $d$ is the gap between two levels. In testing, we control for 500 time steps for evaluation. 

\subsubsection{Ship collision avoidance}
We control a ship (modeled in~\cite{fossen2000survey}(Sec. 4.2)) to avoid obstacles with varied radii. 
%The challenge is that the ship has a large inertia, so it needs to plan ahead when it senses the obstacle.
% x, y, phi, u, v, r, dx1, dy1, r1, dx2, dy2, r2, dx3, dy3, r3
The 12-dim system state has $x,y,\psi,u,v,r$ to describe the pose and $ x_1, y_1, r_1$ ($x_2, y_2, r_2$) to denote the (second) closest obstacle with radius $r_1$ ($r_2$) and relative position $x_1, y_1$ ($x_2, y_2$). The controls are thrust $T$ and rudder angle $\delta$. The dynamics \rbt{are}:
% \begin{equation}
%     \dot{\begin{pmatrix}
%         x\\
%         y\\
%         \psi\\
%         u\\
%         v\\
%         r\\
%         x_1\\
%         y_1\\
%         r_1\\
%         x_2\\
%         y_2\\
%         r_2\\
%     \end{pmatrix}}=\begin{pmatrix}
%         u\cos\psi - v\sin\psi\\
%         u\sin\psi + v\cos\psi\\
%         r\\
%         T\\
%         0.01\delta\\
%         0.5\delta\\
%         0 \\ 0 \\ 0 \\ 0 \\ 0 \\ 0
%     \end{pmatrix}
% \end{equation}
% \begin{equation}
%     \dot{\begin{pmatrix}
%         x\\
%         y\\
%         \psi\\
%         u\\
%         v\\
%         r\\
%     \end{pmatrix}}=\begin{pmatrix}
%         u\cos\psi - v\sin\psi\\
%         u\sin\psi + v\cos\psi\\
%         r\\
%         T\\
%         0.01\delta\\
%         0.5\delta\\
%     \end{pmatrix}
% \end{equation}
% \begin{equation}
% \begin{aligned}
%     \dot{x}=u\cos\psi - v\sin\psi,\quad\dot{y}=u\sin\psi + v\cos\psi\\
%     \dot{\psi}=r,\quad\dot{u}=T,\quad\dot{v}=0.01\delta,\quad\dot{r}=0.5\delta
% \end{aligned}
% \label{eq:ship}
%     % \dot{\begin{pmatrix}
%     %     x\\
%     %     y\\
%     %     \psi\\
%     %     u\\
%     %     v\\
%     %     r\\
%     % \end{pmatrix}}=\begin{pmatrix}
%     %     u\cos\psi - v\sin\psi\\
%     %     u\sin\psi + v\cos\psi\\
%     %     r\\
%     %     T\\
%     %     0.01\delta\\
%     %     0.5\delta\\
%     % \end{pmatrix}
% \end{equation}
$\dot{x}=u\cos\psi - v\sin\psi,\,\dot{y}=u\sin\psi + v\cos\psi, \, \dot{\psi}=r,\,\dot{u}=T,\,\dot{v}=0.01\delta,\,\dot{r}=0.5\delta$.
The rules are: (1) always be in the river (2) always avoid obstacles. The STL is
$
% \begin{equation}
    \Phi=\phi_1 \land \phi_2 \land \phi_3
% \end{equation}
$:
\begin{equation}
    \begin{aligned}
        &\phi_1 = G_{[0,T]}(|y|<D/2)\\
        &\phi_2 = G_{[0,T]}((x-x_1)^2+(y-y_1)^2\leq r_1^2\\
        &\phi_3 = G_{[0,T]}((x-x_2)^2+(y-y_2)^2\leq r_2^2\\
    \end{aligned}
\end{equation}
where $D$ is the width of the river. In testing, we roll out 20 trajectories for 200 steps to evaluate the performance.

\subsubsection{Ship safe centerline tracking}
\rbt{Besides collision avoidance}, the ship must also not deviate more than c time units from the centerline. The 10-dim state now has $x,y,\psi,u,v,r$ to denote the ship state, $x_1,y_1,r_1$ for the closest front obstacle, and $\tau$ for the remaining time the ship can deviate, with
% \begin{equation}
%     \dot{\begin{pmatrix}
%         x\\
%         y\\
%         \psi\\
%         u\\
%         v\\
%         r\\
%         x_1\\
%         y_1\\
%         r_1\\
%         \tau
%     \end{pmatrix}}=\begin{pmatrix}
%         u\cos\psi - v\sin\psi\\
%         u\sin\psi + v\cos\psi\\
%         r\\
%         T\\
%         0.01\delta\\
%         0.5\delta\\
%         0 \\ 0 \\ 0 \\ -\mathbbm{1}(|y|>\gamma)
%     \end{pmatrix}
% \end{equation}
% \begin{equation}
%         \dot{\tau}=-\mathbbm{1}(|y|>\gamma)
% \end{equation}
$\dot{\tau}=-\mathds{1}(|y|>\gamma)$
where $\gamma$ is the deviation threshold. $x_1,y_1,r_1$ and $\tau$ will get reset once the ship passes the current obstacle. The STL is:
$
% \begin{equation}
    \Phi=\phi_1 \land \phi_2 \land \phi_3
% \end{equation}
$
where,
\begin{equation}
    \begin{aligned}
        &\phi_1 = G_{[0,T]}(|y|<D/2)\\
        &\phi_2 = G_{[0,T]}((x-x_1)^2+(y-y_1)^2\leq r_1^2\\
        &\phi_3 = \left(\tau>0\right)U_{[0,T]}\left(G_{[0,T]}(|y|<\gamma)\right)
    \end{aligned}
\end{equation}
We rollout 20 trajectories for 200 steps for evaluation.

\subsubsection{Robot navigation}
\rbt{A battery-powered robot navigates to reach the destinations and charging stations. The state $(x, y, x_d, y_d, x_c, y_c, \tau_b, \tau_s)^T$ denotes the robot, the target, the charging station, the battery, and the remaining time at the charging station. The controls are speed $v$ and heading $\theta$. The dynamics are: $\dot{x}=v\cos\theta,\, \dot{y}=v\sin\theta\rbt{,} \,
    \tau_b=-1,\, \tau_s=-\mathds{1}\{\text{station}\}$} 
% \begin{equation}
%     \dot{\begin{pmatrix}
%         x \\ y \\ x_d \\ y_d \\ x_c \\ y_c \\ \tau_b \\ \tau_s
%     \end{pmatrix}} = \begin{pmatrix}
%         v\cos\theta \\ 
%         v\sin\theta \\
%         0 \\
%         0 \\
%         0 \\
%         0 \\
%         -1 \\
%         -\mathbbm{1}\{\text{at station}\}
%     \end{pmatrix}
% \end{equation}
% \begin{equation}
% \begin{aligned}
%     \dot{x}=v\cos\theta,\, \dot{y}=v\sin\theta\rbt{,} \,
%     % \tau_b=-1,\, \tau_s=-1\{\text{station}\}
%     \tau_b=-1,\, \tau_s=-\mathds{1}\{\text{station}\}
% \end{aligned}
    % \dot{\begin{pmatrix}
    %     x \\ y \\ x_d \\ y_d \\ x_c \\ y_c \\ \tau_b \\ \tau_s
    % \end{pmatrix}} = \begin{pmatrix}
    %     v\cos\theta \\ 
    %     v\sin\theta \\
    %     0 \\
    %     0 \\
    %     0 \\
    %     0 \\
    %     -1 \\
    %     -\mathbbm{1}\{\text{at station}\}
    % \end{pmatrix}
% \end{equation}
where its battery will reset: $\tau_b^+=T$ once it reaches the charger station and the remaining stay time will get reset $\tau_c^+=c$ once the robot leaves the charging station. The rules are: (1) always avoid obstacles (2) go to the target if high battery (3) if low battery, go to the charging station and stay for c time units and (4) keep the battery level non-negative. The STL is
$
% \begin{equation}
    \Phi=\phi_1 \land \phi_2 \land \phi_3 \land \phi_4 \land \phi_5
% \end{equation}
$:
\begin{equation}
    \begin{aligned}
        &\phi_1 = G_{[0,T]}(\neg \text{In(Obstacles)}) \quad \phi_4 = G_{[0,T]}(\tau_b>0) \\
        &\phi_2 = \tau_b>1 \implies F_{[0,T]}(\text{Near}(x_d,y_d))\\
        &\phi_3 = \tau_b<1 \implies F_{[0,T]}(\text{Near}(x_c,y_c))\\
        % &\phi_4 = G_{[0,T]}(\tau_b>0) \\
        &\phi_5 = \text{Near}(x_c,y_c) \Longrightarrow G_{[0, c]}(\text{Near}(x_c,y_c) \lor \tau_s<0)
    \end{aligned}
\end{equation}
\rbt{where $\text{In(Obstacles)}$ checks if the robot is in any obstacle}, and $\text{\text{Near}(x',y')}=(x-x')^2+(y-y')^2\leq r^2$.
In testing, we constructed a sequence of destinations and five charging stations. After the robot reaches one destination, the next one will show up. The robot can choose any station \rbt{for charging.}

% In this case, we can see that, our trained controller can successfully pick up the passengers and send them to their destinations, while guaranteeing collision-free and battery requirements.

\rbt{\subsubsection{Manipulation} A 7DoF Franka Emika robot aims to reach the goal without collisions or breaking the joints (We use PyBullet for visualization). The state $(q_0...q_6,x,y,z)^T$ denotes the joint angles and the goal location. The dynamics are $\dot{q}_i=u_i, i=0,...,6$ where $u_i$ controls the i-th joint. The obstacle is at $(0.3, 0.3, 0.5)$. The STL is $\Phi=\phi_1 \land \phi_2 \land ... \phi_9$:
\begin{equation}
    \begin{aligned}
        &\phi_1 = F_{[0,T]}((x_e-x)^2+(y_e-y)^2+(z_e-z)^2 <r_{g}^2) \\
        &\phi_2 = G_{[0,T]}((x_e-0.3)^2+(y_e-0.3)^2+(z_e-0.5)^2 >r_{o}^2) \\
        &\phi_{i+3} = G_{[0,T]}(q_i>q_i^{min} \ \land \ q_i<q_i^{max}),\, i=0,...,6\\
        % &\phi_4 = G_{[0,T]}(\tau_b>0) \\
        % &\phi_5 = \text{Near}(x_c,y_c) \implies G_{[0, c]}(\text{Near}(x_c,y_c) \lor \tau_s<0) \\
    \end{aligned}
\end{equation}
where $x_e,y_e,z_e$ is the end effector, $q_i^{min}, q_i^{max}$ are the joint limits and $r_g, r_o$ are the goal / obstacle radii. In evaluation, the arm is asked to reach a sequence of goals in 250 steps.
}

% In this case we compare with the CBF method. Notice that...

\subsection{Training and testing comparisons}
During training, we \rbt{compare} the STL accuracy of our method and the \rbt{model-free RL} baselines (\textbf{RL\textsubscript{R}}, \textbf{RL\textsubscript{S}} and \textbf{RL\textsubscript{A}}). As shown in Fig.~\ref{fig:training-curve}, our method \rbt{in most cases} reaches the highest STL accuracy. For tasks with simple dynamics (Traffic, Reach-n-avoid) or simple STL specifications (Ship-safe), the best  RL baselines can have similar STL accuracy to ours. However, no one RL baseline can consistently outperform the others. For tasks with moderate system complexity and complicated STL specification (Ship-track, navigation and manipulation), our approach will have $10\%\sim40\%$ gain in the STL accuracy. We speculate the gain here is because our approach leverages the system dynamics and the gradient information from the STL formula. This shows the advantage of our approach in policy learning compared to RL methods.

\rbt{In testing, we compare with all baselines and show the result with our backup policy (\textbf{Ours\textsubscript{F}}). As shown in Fig.~\ref{fig:testing-acc}, aside from \textbf{CEM}, \textbf{Ours} can outperform the best baselines by $20\%$ for the ship-track task, $45\%$ for the navigation task and $8\%$ for the manipulation task, and only $3\%$ lower than the best approaches on three tasks with simple dynamics or STL formulas. Compared to \textbf{CEM}, \textbf{Ours} has a slightly lower accuracy on ship safe control, tracking control, and navigation but a much higher accuracy for the rest three tasks.  With the backup policy, \textbf{Ours\textsubscript{F}} achieves the same accuracy as the best RL baselines on Reach-n-avoid and Ship-safe tasks and $1.7\%$ lower on the Traffic task and consistently outperforms \textbf{CEM}}. The high STL accuracy of \textbf{CEM} might be due to the short tasks horizon and low action dimension. \rbt{The inferior performance for \textbf{MPC}, \textbf{STL\textsubscript{M} and \textbf{STL\textsubscript{G}}} might be because the solvers encounter numerical issues and cannot converge or the optimizer stucks at local minimum. \textbf{MBPO} and \textbf{PETS} are worse, which might be because they need careful hyper-parameters tuning and learning the dynamics (for the augmented state space) is hard, especially for high-dimension tasks (e.g., navigation and manipulation)}. As for the computation time, as shown in Fig.~\ref{fig:testing-time}, \rbt{\textbf{Ours} is on par with RL and 10X-100X faster than classic MPC or STL solvers. While \textbf{Ours\textsubscript{F}} is slower due to the backup policy}, it is still 3X faster than \textbf{CEM} and other classical methods.

\begin{figure}[htbp]
\centering
    \begin{subfigure}[t]{0.235\textwidth}
        \includegraphics[width=\textwidth]{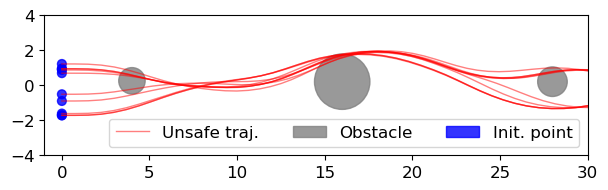}
        \caption{\textbf{Ours} simulation}
    \end{subfigure}
    \begin{subfigure}[t]{0.235\textwidth}
        \includegraphics[width=\textwidth]{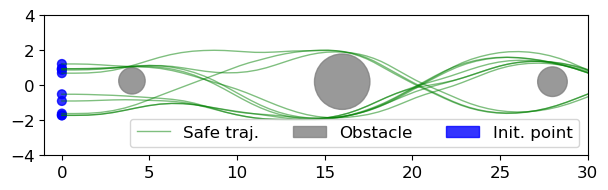}
        \caption{\textbf{Ours\rbt{\textsubscript{F}}} simulation}
    \end{subfigure}
    \\
    \begin{subfigure}[t]{0.235\textwidth}
        \includegraphics[width=\textwidth]{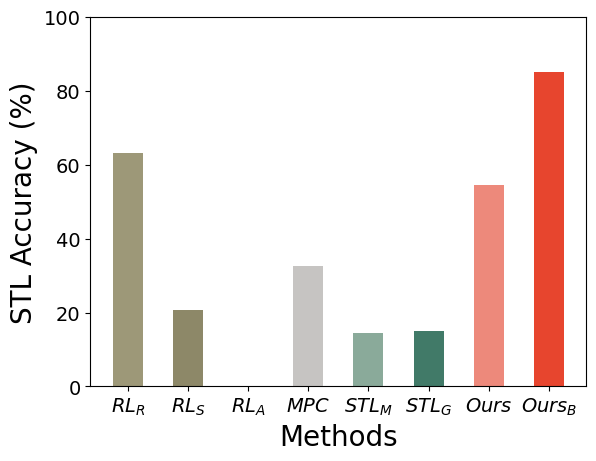}
        \caption{STL Accuracy Comparision}
    \end{subfigure}
    \begin{subfigure}[t]{0.235\textwidth}
        \includegraphics[width=\textwidth]{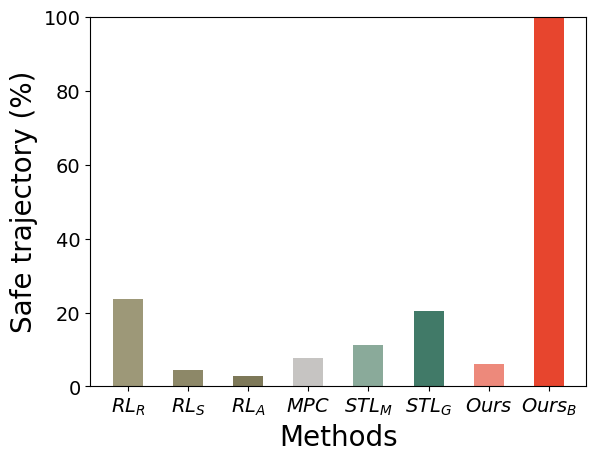}
        \caption{Safety Rate Comparision}
    \end{subfigure}
    \caption{Backup policy in testing}
    \label{fig:backup-policy}
\end{figure}

\subsection{Testing backup policy for out-of-distribution scenarios}
%The test result above all assume the data distribution is the same between training and testing. 
When the testing case is out of distribution (OOD), our proposed backup policy can at least maintain the agent's safety and improve the STL accuracy after recovering from the unseen distribution. \rbt{In the ship-track benchmark}, we shift the first obstacle vertically from the centerline and enlarge the second obstacle, \rbt{which makes the test case OOD (as in training, the obstacles are smaller and on the centerline)}. From Fig.~\ref{fig:backup-policy}(c)(d), we can see that without backup policy, \textbf{Ours} only achieves $6\%$ safety rate and $54\%$ STL accuracy (from Fig.~\ref{fig:backup-policy}(a) we can see that most of the agents will collide with the first obstacle due to OOD). With the backup policy, \rbt{\textbf{Ours\textsubscript{F}}} achieves $100\%$ safety and increases STL accuracy to $85\%$. \rbt{Other baselines are plotted} in Fig.~\ref{fig:backup-policy}(c)(d) for reference.

\setlength{\tabcolsep}{3pt}
\begin{table}[htbp]
\begin{subtable}{.48\linewidth}
\caption{Satisfication threshold $\gamma$}
\centering
% \resizebox{0.4\textwidth}{!}{%
\begin{tabular}{ccc}
\toprule
$\gamma$ & Train Acc. & Val. Acc. \\ 
\midrule
0.0 & 0.855 & 0.788\\
0.1 & 0.920 & 0.912\\
0.2 & 0.918 & 0.915\\
0.5 & 0.957 & 0.957\\
0.8 & 0.957 & 0.958\\
1.0 & 0.958 & 0.955\\
\bottomrule
\end{tabular}
% }
\end{subtable}
\hfill
\begin{subtable}{.48\linewidth}
\centering
\caption{Scaling factor $k$}
% \resizebox{0.4\textwidth}{!}{%
\begin{tabular}{ccc}
\toprule
$k$ & Train Acc. & Val. Acc. \\ 
\midrule
1     & 0.542 & 0.542 \\
10    & 0.950 & 0.946 \\
100   & 0.957 & 0.957 \\
1000  & 0.958 & 0.957 \\
10000 & 0.956 & 0.956 \\
\bottomrule
\end{tabular}
% }
\end{subtable}
\hfill
\begin{subtable}{.48\linewidth}
\centering
\caption{Neural network size \\ (\#neurons x \#layers)}
% \resizebox{0.4\textwidth}{!}{%
\begin{tabular}{ccc}
\toprule
Size & Train Acc. & Val. Acc. \\ 
\midrule
32x2 & 0.951 & 0.954 \\
32x3 & 0.951 & 0.950 \\
64x2 & 0.951 & 0.951 \\
64x3 & 0.951 & 0.956 \\
128x2 & 0.954 & 0.949 \\
128x3 & 0.956 & 0.952 \\
256x2 & 0.952 & 0.951 \\
256x3 & 0.957 & 0.957 \\
\bottomrule
\end{tabular}
% }
\end{subtable}
\hfill
\begin{subtable}{.48\linewidth}
\centering
\caption{Training samples $N$}
% \resizebox{0.4\textwidth}{!}{%
\begin{tabular}{ccc}
\toprule
$N$ & Train Acc. & Val. Acc. \\ 
\midrule
200 & 0.800 & 0.663\\
400 & 0.755 & 0.651\\
800 & 0.941 & 0.836\\
1600 & 0.947 & 0.874\\
3200 & 0.952 & 0.903\\
6400 & 0.962 & 0.920\\
12800 & 0.957 & 0.942\\
25600 & 0.954 & 0.953\\
50000 & 0.957 & 0.957\\
\bottomrule
\end{tabular}
% }
\end{subtable}
\caption{Different setups in training under car benchmark.}
\label{tab:abl_nn}
\end{table}

\subsection{Ablation studies}
Here \rbt{on the traffic benchmark} we show how different parameters and architectures affect the training and the validation accuracy. As shown in Table~\ref{tab:abl_nn}, for hyperparameters such as $\gamma$, $k$ and network size, a wide range of valid values can achieve similar performances ($0.5<\gamma<1.0$, $10<k<10000$, NN from 64x3 to 256x3). As $\gamma$, $k$, the NN size and the training samples increase, the train/validation STL accuracy almost monotonously increases initially and saturates eventually. Another interesting finding is that, with only 800 training samples, we can already achieve $83.6\%$ STL test accuracy, which is slightly higher than \textbf{STL\textsubscript{M}} ($83\%$, the best classical baseline on the Traffic benchmark). This shows that our approach has a high learning efficiency.

\subsection{Limitations}
% Our algorithm does have some limitations. First, our learned policy cannot always satisfy the STL specification due to the generalization error. Although we propose a backup policy to increase STL accuracy, a more robust and time-efficient approach is needed. Besides, we explicitly encode the environmental information into the state representation, which is inefficient for more complicated and dynamic obstacles. Representation learning can be a potential solution to handle this issue. Finally, our method only considers continuous policy space; we will consider discrete policy in future work.

% Our algorithm does have some limitations. 
%TODO mention the challenges in training (gradient descent cannot always find global optimal)
% First, \rbt{our method uses gradient descent to learn control policies, which might be stuck at a local minimum. And in testing,  our learned policy cannot always satisfy the STL specification due to the approximation for the robustness score and the generalization error for unseen samples.} Although we propose a backup policy to increase STL accuracy, a more robust and time-efficient approach is needed. Besides, we explicitly encode environment information into system states, which is inefficient for more complicated and dynamic obstacles. Representation learning can \rbt{potentially} handle this issue. Finally, our method only considers continuous policy; we will consider discrete policy in the future.

First, \rbt{our gradient-based method might be stuck at a local minimum. In testing, our learned policy cannot always satisfy the STL due to the approximation for the robustness score and the generalization error.} Although we propose a backup policy to tackle it, a more robust and time-efficient approach is needed. Besides, we explicitly encode map information into states, which is inefficient for complex and moving obstacles. Representation learning might handle this issue.

\section{Conclusions}

We propose a neural network controller learning framework to fulfill STL specifications in robot tasks. Unlike RL methods, our approach learns the policy directly via gradient descent to maximize the approximated robustness score. Experimental results show that our approach achieves the highest STL accuracy compared to other approaches. A backup policy is proposed for STL monitoring process and guarantees the basic safety. In future, we aim to solve more general STL formulas and perception-based controls.


%%%%%%%%%%%%%%%%%%%%%%%%%%%%%%%%%%%%%%%%%%%%%%%%%%%%%%%%%%%%%%%%%%%%%%%%%%%%%%%%

%%%%%%%%%%%%%%%%%%%%%%%%%%%%%%%%%%%%%%%%%%%%%%%%%%%%%%%%%%%%%%%%%%%%%%%%%%%%%%%%
% \bibliographystyle{plain}
%%%%%%%%%%%%%%%%%%%%%%%%%%%%%%%%%%%%%%%%%%%%%%%%%%%%%%%%%%%%%%%%%%%%%%%%%%%%%%%%
%\section*{Appendix}

\section*{Acknowledgement}
This work was partly supported by the MIT-Ford Alliance Program and National Science Foundation (NSF) CAREER Award \#CCF-2238030. Any opinions, findings, conclusions, or recommendations expressed in this publication are those of the authors and don’t necessarily reflect the views of the sponsors.

\bibliographystyle{ieeetr}
\bibliography{z8_references.bib}

\begin{thebibliography}{10}

\bibitem{sun2022multi}
D.~Sun, J.~Chen, S.~Mitra, and C.~Fan, ``Multi-agent motion planning from
  signal temporal logic specifications,'' {\em IEEE Robotics and Automation
  Letters}, vol.~7, no.~2, pp.~3451--3458, 2022.

\bibitem{dawson2022robust}
C.~Dawson and C.~Fan, ``Robust counterexample-guided optimization for planning
  from differentiable temporal logic,'' in {\em 2022 IEEE/RSJ International
  Conference on Intelligent Robots and Systems (IROS)}, pp.~7205--7212, IEEE,
  2022.

\bibitem{kapoor2020model}
P.~Kapoor, A.~Balakrishnan, and J.~V. Deshmukh, ``Model-based reinforcement
  learning from signal temporal logic specifications,'' {\em arXiv preprint
  arXiv:2011.04950}, 2020.

\bibitem{donze2010robust}
A.~Donz{\'e} and O.~Maler, ``Robust satisfaction of temporal logic over
  real-valued signals,'' in {\em Formal Modeling and Analysis of Timed Systems:
  8th International Conference, FORMATS 2010, Klosterneuburg, Austria,
  September 8-10, 2010. Proceedings 8}, pp.~92--106, Springer, 2010.

\bibitem{pnueli1977temporal}
A.~Pnueli, ``The temporal logic of programs,'' in {\em 18th Annual Symposium on
  Foundations of Computer Science (sfcs 1977)}, pp.~46--57, ieee, 1977.

\bibitem{koymans1990specifying}
R.~Koymans, ``Specifying real-time properties with metric temporal logic,''
  {\em Real-time systems}, vol.~2, no.~4, pp.~255--299, 1990.

\bibitem{maler2004monitoring}
O.~Maler and D.~Nickovic, ``Monitoring temporal properties of continuous
  signals,'' in {\em Formal Techniques, Modelling and Analysis of Timed and
  Fault-Tolerant Systems: Joint International Conferences on Formal Modeling
  and Analysis of Timed Systmes, FORMATS 2004, and Formal Techniques in
  Real-Time and Fault-Tolerant Systems, FTRTFT 2004, Grenoble, France,
  September 22-24, 2004. Proceedings}, pp.~152--166, Springer, 2004.

\bibitem{plaku2016motion}
E.~Plaku and S.~Karaman, ``Motion planning with temporal-logic specifications:
  Progress and challenges,'' {\em AI communications}, vol.~29, no.~1,
  pp.~151--162, 2016.

\bibitem{tabuada2003model}
P.~Tabuada and G.~J. Pappas, ``Model checking ltl over controllable linear
  systems is decidable,'' in {\em HSCC}, vol.~2623, pp.~498--513, Springer,
  2003.

\bibitem{fainekos2009temporal}
G.~E. Fainekos, A.~Girard, H.~Kress-Gazit, and G.~J. Pappas, ``Temporal logic
  motion planning for dynamic robots,'' {\em Automatica}, vol.~45, no.~2,
  pp.~343--352, 2009.

\bibitem{mcmahon2014sampling}
J.~McMahon and E.~Plaku, ``Sampling-based tree search with discrete
  abstractions for motion planning with dynamics and temporal logic,'' in {\em
  2014 IEEE/RSJ International Conference on Intelligent Robots and Systems},
  pp.~3726--3733, IEEE, 2014.

\bibitem{lindemann2018control}
L.~Lindemann and D.~V. Dimarogonas, ``Control barrier functions for signal
  temporal logic tasks,'' {\em IEEE control systems letters}, vol.~3, no.~1,
  pp.~96--101, 2018.

\bibitem{yang2020continuous}
G.~Yang, C.~Belta, and R.~Tron, ``Continuous-time signal temporal logic
  planning with control barrier functions,'' in {\em 2020 American Control
  Conference (ACC)}, pp.~4612--4618, IEEE, 2020.

\bibitem{zhang2023modularized}
Z.~Zhang and S.~Haesaert, ``Modularized control synthesis for complex signal
  temporal logic specifications,'' {\em arXiv preprint arXiv:2303.17086}, 2023.

\bibitem{kantaros2020stylus}
Y.~Kantaros and M.~M. Zavlanos, ``Stylus*: A temporal logic optimal control
  synthesis algorithm for large-scale multi-robot systems,'' {\em The
  International Journal of Robotics Research}, vol.~39, no.~7, pp.~812--836,
  2020.

\bibitem{vasile2017sampling}
C.-I. Vasile, V.~Raman, and S.~Karaman, ``Sampling-based synthesis of
  maximally-satisfying controllers for temporal logic specifications,'' in {\em
  2017 IEEE/RSJ International Conference on Intelligent Robots and Systems
  (IROS)}, pp.~3840--3847, IEEE, 2017.

\bibitem{vasile2020reactive}
C.~I. Vasile, X.~Li, and C.~Belta, ``Reactive sampling-based path planning with
  temporal logic specifications,'' {\em The International Journal of Robotics
  Research}, vol.~39, no.~8, pp.~1002--1028, 2020.

\bibitem{karlsson2020sampling}
J.~Karlsson, F.~S. Barbosa, and J.~Tumova, ``Sampling-based motion planning
  with temporal logic missions and spatial preferences,'' {\em
  IFAC-PapersOnLine}, vol.~53, no.~2, pp.~15537--15543, 2020.

\bibitem{pant2017smooth}
Y.~V. Pant, H.~Abbas, and R.~Mangharam, ``Smooth operator: Control using the
  smooth robustness of temporal logic,'' in {\em 2017 IEEE Conference on
  Control Technology and Applications (CCTA)}, pp.~1235--1240, IEEE, 2017.

\bibitem{leung2019backpropagation}
K.~Leung, N.~Ar{\'e}chiga, and M.~Pavone, ``Backpropagation for parametric
  stl,'' in {\em 2019 IEEE Intelligent Vehicles Symposium (IV)}, pp.~185--192,
  IEEE, 2019.

\bibitem{pantazides2022satellite}
A.~Pantazides, D.~Aksaray, and D.~Gebre-Egziabher, ``Satellite mission planning
  with signal temporal logic specifications,'' in {\em AIAA SCITECH 2022
  Forum}, p.~1091, 2022.

\bibitem{aksaray2016q}
D.~Aksaray, A.~Jones, Z.~Kong, M.~Schwager, and C.~Belta, ``Q-learning for
  robust satisfaction of signal temporal logic specifications,'' in {\em 2016
  IEEE 55th Conference on Decision and Control (CDC)}, pp.~6565--6570, IEEE,
  2016.

\bibitem{li2017reinforcement}
X.~Li, C.-I. Vasile, and C.~Belta, ``Reinforcement learning with temporal logic
  rewards,'' in {\em 2017 IEEE/RSJ International Conference on Intelligent
  Robots and Systems (IROS)}, pp.~3834--3839, IEEE, 2017.

\bibitem{balakrishnan2019structured}
A.~Balakrishnan and J.~V. Deshmukh, ``Structured reward shaping using signal
  temporal logic specifications,'' in {\em 2019 IEEE/RSJ International
  Conference on Intelligent Robots and Systems (IROS)}, pp.~3481--3486, IEEE,
  2019.

\bibitem{cho2018learning}
K.~Cho and S.~Oh, ``Learning-based model predictive control under signal
  temporal logic specifications,'' in {\em 2018 IEEE International Conference
  on Robotics and Automation (ICRA)}, pp.~7322--7329, IEEE, 2018.

\bibitem{cohen2021model}
M.~H. Cohen and C.~Belta, ``Model-based reinforcement learning for approximate
  optimal control with temporal logic specifications,'' in {\em Proceedings of
  the 24th International Conference on Hybrid Systems: Computation and
  Control}, pp.~1--11, 2021.

\bibitem{puranic2021learning}
A.~G. Puranic, J.~V. Deshmukh, and S.~Nikolaidis, ``Learning from
  demonstrations using signal temporal logic in stochastic and continuous
  domains,'' {\em IEEE Robotics and Automation Letters}, vol.~6, no.~4,
  pp.~6250--6257, 2021.

\bibitem{liu2021recurrent}
W.~Liu, N.~Mehdipour, and C.~Belta, ``Recurrent neural network controllers for
  signal temporal logic specifications subject to safety constraints,'' {\em
  IEEE Control Systems Letters}, vol.~6, pp.~91--96, 2021.

\bibitem{hashimoto2022stl2vec}
W.~Hashimoto, K.~Hashimoto, and S.~Takai, ``Stl2vec: Signal temporal logic
  embeddings for control synthesis with recurrent neural networks,'' {\em IEEE
  Robotics and Automation Letters}, vol.~7, no.~2, pp.~5246--5253, 2022.

\bibitem{donze2013efficient}
A.~Donz{\'e}, T.~Ferrere, and O.~Maler, ``Efficient robust monitoring for
  stl,'' in {\em Computer Aided Verification: 25th International Conference,
  CAV 2013, Saint Petersburg, Russia, July 13-19, 2013. Proceedings 25},
  pp.~264--279, Springer, 2013.

\bibitem{haarnoja2018soft}
T.~Haarnoja, A.~Zhou, P.~Abbeel, and S.~Levine, ``Soft actor-critic: Off-policy
  maximum entropy deep reinforcement learning with a stochastic actor,'' in
  {\em International conference on machine learning}, pp.~1861--1870, PMLR,
  2018.

\bibitem{raffin2021stable}
A.~Raffin, A.~Hill, A.~Gleave, A.~Kanervisto, M.~Ernestus, and N.~Dormann,
  ``Stable-baselines3: Reliable reinforcement learning implementations,'' {\em
  The Journal of Machine Learning Research}, vol.~22, no.~1, pp.~12348--12355,
  2021.

\bibitem{janner2019trust}
M.~Janner, J.~Fu, M.~Zhang, and S.~Levine, ``When to trust your model:
  Model-based policy optimization,'' {\em Advances in neural information
  processing systems}, vol.~32, 2019.

\bibitem{chua2018deep}
K.~Chua, R.~Calandra, R.~McAllister, and S.~Levine, ``Deep reinforcement
  learning in a handful of trials using probabilistic dynamics models,'' {\em
  Advances in neural information processing systems}, vol.~31, 2018.

\bibitem{de2005tutorial}
P.-T. De~Boer, D.~P. Kroese, S.~Mannor, and R.~Y. Rubinstein, ``A tutorial on
  the cross-entropy method,'' {\em Annals of operations research}, vol.~134,
  pp.~19--67, 2005.

\bibitem{andersson2019casadi}
J.~A. Andersson, J.~Gillis, G.~Horn, J.~B. Rawlings, and M.~Diehl, ``Casadi: a
  software framework for nonlinear optimization and optimal control,'' {\em
  Mathematical Programming Computation}, vol.~11, pp.~1--36, 2019.

\bibitem{gurobi2021gurobi}
L.~Gurobi~Optimization, ``Gurobi optimizer reference manual,'' 2021.

\bibitem{kingma2014adam}
D.~P. Kingma and J.~Ba, ``Adam: A method for stochastic optimization,'' {\em
  arXiv preprint arXiv:1412.6980}, 2014.

\bibitem{paszke2019pytorch}
A.~Paszke, S.~Gross, F.~Massa, A.~Lerer, J.~Bradbury, G.~Chanan, T.~Killeen,
  Z.~Lin, N.~Gimelshein, L.~Antiga, {\em et~al.}, ``Pytorch: An imperative
  style, high-performance deep learning library,'' {\em Advances in neural
  information processing systems}, vol.~32, 2019.

\bibitem{fossen2000survey}
T.~I. Fossen, ``A survey on nonlinear ship control: From theory to practice,''
  {\em IFAC Proceedings Volumes}, vol.~33, no.~21, pp.~1--16, 2000.

\end{thebibliography}
% \clearpage
% \input{rebuttal/zz_statement_of_changes}
\end{document}